\documentclass[accepted]{article}
\pdfoutput=1
\usepackage{aistats2018}

\usepackage{xr}
\usepackage{caption}
\usepackage{xspace}
\usepackage{wrapfig}
\usepackage[utf8]{inputenc} % allow utf-8 input
\usepackage[T1]{fontenc}    % use 8-bit T1 fonts
\usepackage{hyperref}       % hyperlinks
\usepackage{url}            % simple URL typesetting
\usepackage{array}
\usepackage{booktabs}       % professional-quality tables
\usepackage{amsfonts}       % blackboard math symbols
\usepackage{nicefrac}       % compact symbols for 1/2, etc.
\usepackage{microtype}      % microtypography
\usepackage{graphicx}
\usepackage{subcaption}
\usepackage{color, comment}
\usepackage{bm, stmaryrd}
\usepackage{pifont}
\usepackage{algorithm}
\usepackage{algorithmic}
\usepackage{footnote}
\usepackage{tablefootnote}
\usepackage{amsmath, amsthm, amssymb}

\newtheorem{theorem}{Theorem}
\newtheorem{corollary}{Corollary}

\newtheorem{lemma}{Lemma}

\newtheorem{proposition}{Proposition}
\newtheorem{remark}{Remark}

\newcommand{\bR}{\mathbb{R}} %real numbers
\newcommand{\bE}{\mathbb{E}}

\newcommand{\tr}{\text{trace}}
\newcommand{\mmax}{m_{\max}}
\newcommand{\mmin}{m_{\min}}

\newcommand{\bfone}{\textbf{1}}
\newcommand{\xhat}{\hat{X}}
\newcommand{\cl}{r}

\newcommand{\innerprod}[2]{\langle {#1}, {#2}\rangle}

\newcommand{\bas}[1]{\begin{align*}#1\end{align*}}
\newcommand{\ba}[1]{\begin{align}#1\end{align}}
\newcommand{\beq}[1]{\begin{equation}#1\end{equation}}
\newcommand{\bsplt}[1]{\begin{split}#1\end{split}}

\newcommand{\sdp}{SPUR\xspace}
\newcommand{\lnuc}{L1+nuc\xspace}

\newcommand{\diag}{\text{diag}}

\newcommand{\bk}{\color{black}}
\newcommand{\bmv}{\bm{m}}
\newcommand{\vertiii}[1]{{\left\vert\kern-0.25ex\left\vert\kern-0.25ex\left\vert #1 
    \right\vert\kern-0.25ex\right\vert\kern-0.25ex\right\vert}}
\definecolor{alizarin}{rgb}{0.82, 0.1, 0.26}
 
\DeclareCaptionLabelFormat{andtable}{#1~#2  \&  \tablename~\thetable}

\begin{document}

\twocolumn[
\aistatstitle{Provable Estimation of the Number of Blocks in Block Models}
%\aistatsauthor{ Author 1 \And Author 2 \And  Author 3 }
%\aistatsaddress{ Institution 1 \And  Institution 2 \And Institution 3 } ]
\aistatsauthor{ Bowei Yan \And Purnamrita Sarkar \And  Xiuyuan Cheng }
\aistatsaddress{ University of Texas at Austin \And  University of Texas at Austin \And Duke University } ]

\begin{abstract}
	Community detection is a fundamental unsupervised learning problem for unlabeled networks which has a broad range of applications. Many community detection algorithms assume that the number of clusters $r$ is known apriori. %While provable algorithms for finding $r$ have recently garnered much attention from the theoretical statistics community, existing methods often make strong model assumptions about the separation between clusters or communities.
	 In this paper, we propose an approach based on semi-definite relaxations, which does not require prior knowledge of model parameters like many existing convex relaxation methods and recovers the number of clusters and the clustering matrix exactly under a broad parameter regime, with probability tending to one.  On a variety of simulated and real data experiments, we show that the proposed method often outperforms  state-of-the-art techniques for estimating the number of clusters. 
\end{abstract}

\section{Introduction}
As a fundamental problem in network analysis, community detection has drawn much attention from both theorists and practitioners. Most existing methods require the prior knowledge of the true number of clusters, which is often unavailable in real data applications. 
In this paper we mainly focus on provably estimating the number of clusters in a network.

While it is tempting to use a two-stage procedure~\cite{chen2014improved} where the number of clusters is estimated first and then used  as an input for clustering, an erroneous estimation on the number of clusters can deteriorate the clustering accuracy. Instead, we design an algorithm which  estimates the true number of clusters and recovers the cluster memberships simultaneously, with provable guarantees. 

In this paper, we focus on the widely-used Stochastic Block Models (SBM)~\cite{holland1983stochastic}. The model assumes the probability of an edge between two nodes are completely determined by the unknown cluster memberships of the nodes. Essentially, this imposes stochastic equivalence, i.e. all nodes in the same cluster behave identically in a probabilistic sense. Despite its simplicity, the SBM is used as a building block in more sophisticated models like the Degree Corrected Block Models~\cite{karrer2011stochastic} and Mixed Membership Block Models~\cite{airoldi2008mixed} and has been applied successfully for clustering real world networks.

Semi-definite programming (SDP) relaxations for network clustering have been widely studied and many different formulations have been proposed. It has been empirically observed that these methods have better clustering performance compared to spectral methods~\cite{amini2014semidefinite, yan2016robustness, chen2014improved}. As shown by~\cite{chen2012clustering,amini2014semidefinite}, SDPs arise naturally when the likelihood of a SBM with equal cluster sizes is relaxed. SDP returns a relaxation of the clustering matrix, which is a $n\times n$ ($n$ being the number of nodes) symmetric matrix whose $ij^{th}$ element is one if nodes $i$ and $j$ belong to the same cluster and zero otherwise.  We present a detailed discussion on related work in Section \ref{sec:related}. In this work, we use the SDP formulation proposed by \cite{peng2007approximating}, which uses a normalized variant of the clustering matrix. Similar relaxations have been used to study $k$-means clustering for sub-gaussian mixtures \cite{mixon2016clustering} and SBMs~\cite{yan2016convex}.

For community detection in SBM, an algorithm is considered effective if it is asymptotically consistent.
There are two types of consistency in the literature. When the number of nodes in the graph is large enough, the network is sufficiently dense, and the signal (usually defined by the separation between intra-cluster probability and inter-cluster probability) is strong enough, strongly consistent methods recover the ground truth labels exactly, while the weakly consistent methods recover a fraction of labels correctly where the fraction approaches one as $n$ goes to infinity.

There have been a number of SDP relaxations for general unbalanced cluster sizes which have been shown to be strongly consistent~\cite{perry2015semidefinite,hajek2016achievinga,cai2015robust}. One can argue that these methods readily render themselves to estimation of the number of blocks $r$. The idea would be to run the SDP with different values of  $r$, and for the correct one the clustering matrix will be the true clustering matrix with high probability. However, all these methods require the knowledge of model parameters. Furthermore, they work in the unequal cluster size setting by introducing an additional penalty term, which requires further tuning. Hence  each run with a different choice of $r$ would have an internal tuning step adding to the already expensive computation of the SDP. In this paper, we propose a formulation that is a) entirely tuning free when the number of clusters is known, and b) when it is unknown, is able to recover the number of clusters and the clustering matrix in one shot.
\bk

%It is worth pointing out here the two essential strengths of our consistency result compared to existing SDP relaxations. First, this SDP formulation is entirely agnostic to model parameters unlike many similar formulations for unequal cluster sizes in previous work~\cite{yan2016convex,perry2015semidefinite,hajek2016achievinga,cai2015robust}. Also, when the number of clusters is known, existing formulations work in unequal cluster size settings by introducing an additional penalty term, which requires further tuning. In contrast, when the number of clusters is known, we do not require tuning.

Furthermore, our method provably works in the weakly assortative setting, whereas 
%only has weak consistency established. In this paper, we will see that it actually achieves strong consistency, i.e. exact recovery under a broad parameter range. 
 the usual necessary separation condition for recovery is that the maximal inter-cluster connecting probability (think of this as noise) is smaller than the minimal intra-cluster connecting probability (the signal) by a certain margin. This separation condition is known as strong assortativity. In contrast, our work only requires that for each node, the probability of connecting to the nodes in its own cluster is greater by a margin than the largest probability of connecting with nodes in other clusters. This property is called weak assortativity. It is not hard to see that weakly assortative models are a superset of strongly assortative models. Weak assortativity was first introduced in \cite{amini2014semidefinite}, who establish exact recovery under this weaker condition for SDPs for blockmodels with \textit{equal sized} communities.

In Sec~\ref{sec:exp} we sketch a rather interesting empirical property of our algorithm (also pointed out in~\cite{perry2015semidefinite}); namely it can identify different granularities of separations as a byproduct. For example, the tuning phase, which we sketch in Section~\ref{sec:exp}, finds different substructures of the network as it searches over different  tuning parameters. For example, if there are $K$ meta clusters which are more well separated than the rest, then as we tune, we will first find these meta-clusters, and then finer substructures within them. While this is not the main goal of our paper, it indeed makes our approach ideal for exploratory analysis of networks. We also leave the theoretical analysis of finding multi-resolution clusterings for future work. 

We will formalize these concepts in Section~\ref{sec:setup} and discuss the related work in more detail in Section~\ref{sec:related}. Section~\ref{sec:mainres} contains our main theoretical contributions and finally, in Section~\ref{sec:exp} we demonstrate the efficacy of our algorithm compared to existing methods on a variety of simulated and real networks. 

\section{Problem Setup and Notations}
\label{sec:setup}
Assume $(S_1,\cdots, S_{\cl})$ represent a $\cl$-partition for $n$ nodes $\{1,\cdots,n\}$. Let $m_i=|S_i|$ be the size of each cluster, and let  $m_{\min}$ and $m_{\max}$ be the minimum and maximum cluster sizes respectively. %denote $M=\diag(m_1,\cdots, m_r)$.
We denote by $A$ the $n\times n$ binary adjacency matrix with the true and unknown membership matrix $Z=\{0,1\}^{n\times \cl}$,
\begin{align}
\label{eq:data-model}
\tag{SBM($B,Z$)} &P(A_{ij}=1|Z) = Z_i^TBZ_j  \quad \forall i\ne j, \\
&P(A_{ii}=0)=1,\quad Z^TZ=diag(\bmv),
 \end{align}
where $B$ is a $\cl \times \cl$ matrix of within and across cluster connection probabilities and $\bmv$ is a length $r$ vector of cluster sizes. The elements of $B$ can decay with graph size $n$. In this paper we focus on the regime where the average expected degree grows faster than logarithm of $n$. In this regime, it is possible to obtain strong or weak consistency.

Given any block model, the goal for community detection is to recover the column space of $Z$.  For example if we can solve $ZZ^T$ or its normalized variant $Zdiag(\bmv)^{-1}Z^T$, then the labels can be recovered from the eigenvectors of the clustering matrix. 
\paragraph{The normalized clustering matrix:}
In this paper we focus on recovering the following normalized version:
\vspace{-.5em}
\ba{
X_0=Z\diag(\bmv)^{-1}Z^T\label{eq:x0}
}
It can be easily checked that  $X_0\bfone_n =\bfone_n$, since $Z\bfone_k=\bfone_n$. Furthermore, $X_0$ is positive semi-definite and its trace (which equals its nuclear norm as well) equals the number of clusters $r$. 

\paragraph{Assortativity (strong vs.\ weak):}
Assortativity is a condition usually required in membership recovery. The strong assortativity (see Eq.~\eqref{eq:strong-assortativity})  requires the smallest diagonal entry to be greater than the largest off-diagonal entry.
\begin{align}
\min_{k} B_{kk}-\max_{k\ne \ell} B_{k\ell}&>0\label{eq:strong-assortativity}\\
\min_{k} \left( B_{kk}-\max_{\ell \ne k} B_{k\ell} \right) &>0.\label{eq:weak-assortativity}
\end{align}
%	\begin{equation}
%\min_{k} B_{kk}-\max_{k\ne \ell} B_{k\ell}>0.
%\label{eq:strong-assortativity}
%	\end{equation}
%	\begin{equation}
%\min_{k} \left( B_{kk}-\max_{\ell \ne k} B_{k\ell} \right) >0.
%\label{eq:weak-assortativity}
%	\end{equation}
 \cite{amini2014semidefinite} first introduces an SDP that provably achieves exact recovery for weakly assortative models (Eq.~\eqref{eq:weak-assortativity}) with \textit{equal cluster sizes}, i.e., 
compared with \eqref{eq:strong-assortativity}, weak assortativity only compares the probability within the same row and column; it requires that any given cluster $k$, should have a larger probability of connecting within itself than with nodes in any other cluster. It is easy to check that strong assortativity indicates weak assortativity and not vice versa. 

%\subsection{Notations}
%Let $S_k$ be the indices of the $k$th community, recall that $|S_k|=m_k$. 
For any matrix $X\in \bR^{n\times n}$, denote $X_{S_kS_\ell}$ as the submatrix of $X$ on indices $S_k\times S_\ell$, and $X_{S_k}:=X_{S_k\times S_k}$. Let $\bfone$ be all one vector, and $\bfone_{S_k}\in \bR^n$ be the indicator vector of $S_k$, equal to one on $S_k$ and zero elsewhere. 
The inner product of two matrices is defined as $\innerprod{A}{B}=\tr(A^TB)$. We use $\circ$ to denote the Schur (elementwise) product of two matrices. Standard notations for complexity analysis $o, O, \Theta, \Omega$ will be used. And those with a tilde are to represent the same order ignoring log factors.

\section{Related Work}
\label{sec:related}

While most community detection methods assume that the number of communities ($r$) is given apriori, there has been much empirical and some theoretical work on estimating $r$ from networks. 

\textbf{Methods for estimating $r$: }
% Likelihood based
A large class of methods chooses $r$ by maximizing some likelihood-based criterion. While there are notable methods for estimating $r$ for \textit{non}-network structured data from mixture models~\cite{pelleg-xmeans,hamerly-gmeans,biernacki2000assessing,Patterson_populationstructure}, we will not discuss them here. 
 
Many likelihood-based methods use variants or approximations of Bayesian Information
Criterion (BIC); BIC, while a popular choice for model selection, can be computationally expensive since  it depends on the likelihood of the observed data. Variants of the Integrated Classification Likelihood (ICL, originally proposed by~\cite{biernacki2000assessing}) were proposed in \cite{daudin2008mixture, latouche2012variational}. Other BIC type criteria are studied in~\cite{mariadassou2010uncovering, saldana2017many,mcauley2012learning}. 

% We now discuss Bayesian inference methods \cite{hofman2008bayesian, riolo2017efficient}.  
In~\cite{hofman2008bayesian} a computationally efficient variational Bayes technique is proposed to estimate $r$. This method is empirically shown to be more accurate than BIC and ICL and faster than Cross Validation based approaches~\cite{chen2016network}. %Experiment on laptop, $N=10^6$  ~ 6min. Not provable.
A Bayesian approach with a new prior and an efficient sampling scheme is used  to estimate $r$ in~\cite{riolo2017efficient}.  While the above methods are not provable, a provably consistent likelihood ratio test is proposed to estimate $r$ in~\cite{wang2015likelihood}. %In this paper,  we will compare our approach with existing techniques which are provable under a given generative model. %However the method is computationally extensive. 

Another class of methods is based on the spectral approach. The idea is to estimate $r$ by the number of ``leading eigenvalues'' of a suitably normalized adjacency matrix ~\cite{owen2009bi, josse2012selecting, chatterjee2015matrix, fishkind2013consistent}. Of these the USVT estimator~\cite{chatterjee2015matrix} uses random matrix theory to estimate $r$ simply by thresholding the empirical eigenvalues of the adjacency matrix appropriately. 
 In~\cite{bordenave2015non} it is shown that the informative eigenvalues of the non-backtracking matrix are real-valued and separated from the bulk under the SBM. %Therefore one can estimate $r$ by counting the number of real eigenvalues of non-backtracking matrix that are above a certain threshold.  
 In~\cite{le2015estimating}, the spectrum of the non-backtracking matrix and the Bethe-Hessian operator are used to estimate $r$, the later being shown to work better for sparse graphs. 

\cite{abbe2015recovering} proposes a degree-profiling method achieving the optimal information theoretical limit for exact recovery. This agnostic algorithm first learns a preliminary classification based on a subsample of edges, then adjust the classification for each node based on the degree-profiling from the preliminary classification. However it involves a highly-tuned and hard to implement spectral clustering step (also noted by \cite{perry2015semidefinite}). It also requires specific modifications when applied to real world networks (as pointed out by the authors) .

In~\cite{zhao2011community}, communities are sequentially extracted from a network; the stopping criterion uses a bootstrapped approximation of the null distribution of the statistic of choice. In ~\cite{bickel2016hypothesis}, the null distribution of a spectral test statistic is derived, which is used to test $r = 1$ vs $r > 1$ at each step of a recursive bipartitioning algorithm. %The test is based on the asymptotic null distribution of the largest eigenvalue of a suitably scaled and centered adjacency matrix, which the authors derive. Since the convergence to this asymptotic null distribution is slow, a bootstrap correction is needed in practice. 
A generalization of this approach for testing a null hypothesis for $r$ blocks can be found in~\cite{lei2016goodness}. While the algorithm in~\cite{bickel2016hypothesis} often produces 
over-estimates of $r$,~\cite{lei2016goodness}'s hypothesis test depends on a preliminary fitting with an algorithm which exactly recovers the parameters. The final accuracy heavily depends on the accuracy of this fit.
%The drawback of this method is that the 
Network cross-validation based methods have also been used for selecting $r$. The cross-validation can be carried out either via node splitting~\cite{airoldi2008mixed}, or node-pair splitting~\cite{hoff2008modeling, chen2016network}; the asymptotic consistency of these methods are shown in~\cite{chen2016network}. 
We conclude with a comparison of our approach to other convex relaxations.

\paragraph{Comparison to other convex relaxations}
In recent years, SDP has drawn much attention in handling community detection problems with Stochastic Block Models. Various of relaxations have been shown to possess strong theoretical guarantees in recovering the true clustering structure without rounding~\cite{amini2014semidefinite, hajek2016achievinga, hajek2016achievingb, cai2015robust,perry2015semidefinite,montanari2015semidefinite, guedon2014community}. Most of them aim at recovering a binary clustering matrix, and show that the relaxed SDP will have the ground truth clustering matrix as its unique optimal solution.
For unbalanced cluster sizes, an extra penalization is often introduced which requires additional tuning~\cite{cai2015robust, hajek2016achievinga, perry2015semidefinite}. While one can try different choices of $r$ for these SDPs until achieving exact recovery, the procedure is slower since each run would need another internal tuning step.

SDP with a normalized clustering matrix was introduced by~\cite{peng2007approximating}. They have been used for network clustering~\cite{yan2016convex} and for the relaxation of $k$-means clustering of non-network structured data~\cite{peng2007approximating, mixon2016clustering} . 
\beq{\tag{SDP-PW}
\bsplt{
\max &\quad  \innerprod{A}{X}\\
s.t. &\quad X\succeq 0, X\ge 0, \\
& \quad X\bfone=\bfone, \tr(X)=r
}
\label{eq:sdp-knownk}
}
 However the formulation in \cite{yan2016convex} requires an additional parameter as an lower bound on the minimum size of the clusters; loose lower bounds can empirically deteriorate the performance. Also the authors only establish weak consistency of the solution.

Some of these methods do not require the knowledge of $r$ in the constraints, but instead have the dependency implicitly. In~\cite{chen2015convexified} a convexified modularity minimization for Degree-corrected SBM is proposed, which also works for SBMs as a special case of degree corrected models. The authors suggest one over total number of edges as the default value for the tuning parameter, but when dealing with delicate structures of the network, this suggested value can be sub-optimal and further tuning is required. The procedure also requires $r$ for the final clustering of the nodes via Spectral Clustering from the clustering matrix.

A different convex relaxation motivated by low-rank matrix recovery is studied in~\cite{chen2014improved}. Here, first the eigenspectrum of $A$ is used to estimate $r$, which is subsequently used to estimate tuning parameters required in the main algorithm. We can also tune the tuning parameter with other heuristics, but as the theorem in that paper implies, the tuning parameter needs to lie between the minimal intra-cluster probabilities and maximal inter-cluster probabilities, which is only feasible for strongly assortative settings. We provide more details in the experimental section.

\paragraph{Hierarchical clustering structures}
A phenomenon that has been observed \cite{perry2015semidefinite, chen2014improved} is that convex relaxations can be used to find hierarchical structures in the networks by varying the tuning parameter.  In the experimental section we demonstrate this with some examples.

\paragraph{Separation conditions}
In terms of the separation conditions, most aforementioned convex relaxations are consistent in the dense regime under {\it strong assortativity} except~\cite{amini2014semidefinite} and~\cite{yan2016convex}. However, \cite{amini2014semidefinite} only prove exact recovery of clusters for equal sized clusters, whereas~\cite{yan2016convex} only show weak consistency and require the knowledge of additional parameters like the minimum cluster size. \cite{perry2015semidefinite} shows exact recovery while matching the information theoretical lower bound, which is not the goal of this paper. %However, that result only holds for the restricted scenario where all inter-cluster probabilities are the same, so are the intra-cluster probabilities.

In this paper, we compare our algorithm with noted representatives from the related work. From the Spectral methods, we compare with the USVT estimator and the Bethe Hessian based estimator~\cite{le2015estimating}, which has been shown to empirically outperform a variety of other provable techniques like~\cite{wang2015likelihood} and~\cite{chen2016network}. For these methods, we first estimate $r$ and then use the Regularized Spectral Clustering~\cite{amini2013pseudo,le2015sparse} algorithm to obtain the final clustering.
From the convex relaxation literature, we compare with~\cite{chen2014improved} and~\cite{chen2015convexified}, neither of which require $r$ for estimating the clustering  except for the final clustering step. 
\section{Main Result}
\label{sec:mainres}
%We present the main result in this section. 
In various SDP relaxations for community detection under SBMs, the objective function is taken as the linear inner product of the adjacency matrix $A$ and the target clustering matrix $X$, some formulations also have some additional penalty terms. The inner product objective can be derived from several different metrics for the opitimality of the clustering, such as likelihood or modularity. The penalty terms vary depending on what kind of a solution the SDP is encouraged to yield.
For example, in low-rank matrix recovery literature, it is common practice to use the nuclear norm regularization to encourage low-rank solution. For a positive semi-definite matrix, the nuclear norm is identical to its trace.
When the number of clusters $r$ is unknown,  we consider the following SDP.
\beq{
\bsplt{
\max & \quad \tr(AX)-\lambda \tr(X) \\
s.t. & \quad X\succeq 0, X\ge 0, X\bfone=\bfone,
}
\label{eq:sdp-lambda}\tag{SDP-$\lambda$}
}
where $\lambda$ is a tuning parameter, and $X\ge 0$ is an element-wise non-negativity constraint.
The following theorem guarantees the exact recovery of the ground truth solution matrix, when $\lambda$ lies in the given range for the tuning parameter.

\begin{theorem}
Let $\hat{X}$ be the optimal solution of \eqref{eq:sdp-lambda} for $A\sim \text{SBM}(B,Z)$ where $\mmin$ and $\mmax$ denote the smallest and largest cluster sizes respectively. Define the separation parameter $\delta = \min_k (B_{kk}-\max_{\ell\ne k}B_{k\ell})$. If 
\bas{
&c_1\max_k\sqrt{m_kB_{kk}}+c_2\sqrt{n\max_{k\ne \ell}B_{k\ell}} \le \lambda\\
&\qquad \le \mmin \left( \delta- \max_{k,\ell}\sqrt{\frac{B_{k\ell}\log m_k}{m_k}}\right)
} then $\hat{X}=X_0$ with probability at least $1-n^{-1}$ provided
\ba{
\delta \ge 2\sqrt{6\log n}\max_k \sqrt{\frac{B_{kk}}{m_k}}+ &6\max_{\ell\ne k}\sqrt{\frac{B_{k\ell}\log n}{m_{\min}}} \nonumber \\
& +\frac{c\sqrt{np_{\max}}}{\mmin}
%\delta&= \tilde{\Omega}\left(\frac{\mmax}{\mmin} \max_k \sqrt{\frac{ \max(B_{kk},r\max_{\ell\neq k}B_{k\ell})}{m_k}}\right)
\label{eq:separation}
}
\label{th:unspecified}
\end{theorem}

\begin{remark}
The above theorem controls how fast the different parameters can grow or decay as $n$ grows. For ease of exposition, we will discuss these constraints on each parameter by fixing the others. The number of clusters  $r$ can increase with $n$. In the dense setting, when $B_{kk}=\Theta(1)$,  %$r/m_k=\Theta(r^2/n)$ and since the separation between probabilities $\delta=o(1)$, 
$\mmin=\omega(\sqrt{n})$ and $r=o(\sqrt{n})$, which matches with the best upper bound on $r$ from existing literature. %For the balance between cluster sizes $\mmax/\mmin$, if $r,B_{kk}=\Theta(1)$,  then %$\mmax/\mmin^{3/2}=\tilde{o}(1)$, hence %$(\mmax/\mmin)^{3/2} =\tilde{o}(\sqrt{n})$ and  $\mmax/\mmin=o(\sqrt{n})$. 
Finally when $\max_k B_{kk} = \Theta(\log n/n)$, we note that $m_{\min}=\tilde{\Theta}(n)$ and $r=\tilde{\Theta}(1)$. 
\end{remark}
We can see from the condition in Theorem \ref{th:unspecified} that the tuning parameter should be of the order $\sqrt{d}$ where $d$ is the average degree.
In fact, as shown in the following theorem, when the $\lambda$ is greater than the operator norm, \eqref{eq:sdp-lambda} returns a degenerating rank-1 solution. This gives an upper bound for $\lambda$.

\begin{proposition}
When $\lambda \ge \|A\|_{op}$, then the solution for \eqref{eq:sdp-lambda} is $\bfone\bfone^T/n$.
\label{prop:lambda_upper}
\end{proposition}
The proof of Proposition~\ref{prop:lambda_upper} is to be found in Appendix~\ref{app:lambda_upper}.
Recall the properties of the ground truth clustering matrix defined in Eq.~\eqref{eq:x0}.
If the optimal solution recovers the ground truth $X_0$ exactly, we can estimate $r$ easily from its trace. Therefore we have the following corollary.

%\rd Dont we need to say this is from a blockmodel with r blocks?\bk
\begin{corollary}
Let $\hat{X}$ be the optimal solution of \eqref{eq:sdp-lambda} with $A\sim \text{SBM}(B,Z)$, where $B\in [0,1]^{r\times r}$. Under the condition in Theorem \ref{th:unspecified}, $\tr(\hat{X})=r$ with probability at least $1-n^{-1}$.
\label{cor:estk}
\end{corollary}

In particular, when $r$ is known, we have the following exact recovery guarantee, which is stronger than the weak consistency result in~\cite{yan2016convex}.

\begin{theorem}
Let $A\sim \text{SBM}(B,Z)$, where $B\in [0,1]^{r\times r}$. $X_0$ is the optimal solution of \eqref{eq:sdp-knownk}
with probability at least $1-n^{-1}$, if the separation condition Eq.~\eqref{eq:separation} holds true.
%\vspace{-.5em}
%\bas{
%&\min_k (B_{kk}-\max_{\ell\ne k}B_{k\ell}) \\
%&= \tilde{\Omega}\left(\frac{\mmax}{\mmin} \max_k \sqrt{\frac{ \max(B_{kk},r\max_{\ell\neq k}B_{k\ell})}{m_k}}\right)
%}
\label{th:exact}
\end{theorem}
%\vspace{-1em}
We can see that the two SDPs \eqref{eq:sdp-lambda} and \eqref{eq:sdp-knownk} are closely related. In fact, the Lagrangian function of \eqref{eq:sdp-knownk} is same as the Lagrangian function of \eqref{eq:sdp-lambda} if we take the lagrangian multiplier for the constraint $\tr(X)=r$ as $\lambda$. We use this fact in the proof of Theorem~\ref{th:unspecified}. Both proofs rely on constructing a dual certificate witness, which we elaborate in the following subsection.%Section \ref{sec:dual}.
\subsection{Dual Certificate Witness}
\label{sec:dual}

\begin{figure*}[t]
\begin{tabular}{ccccc}
%\hspace{-3em} 
\includegraphics[width=.2\textwidth]{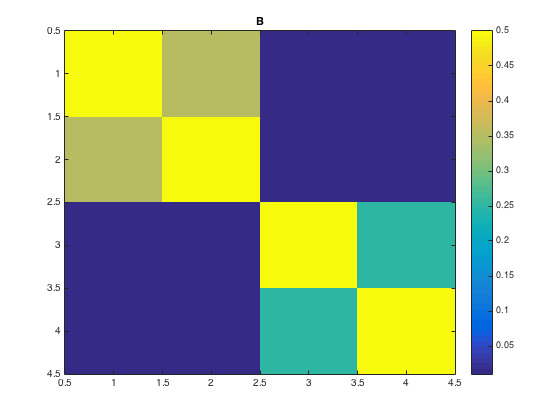} \hspace{-2em}
&\includegraphics[width=.2\textwidth]{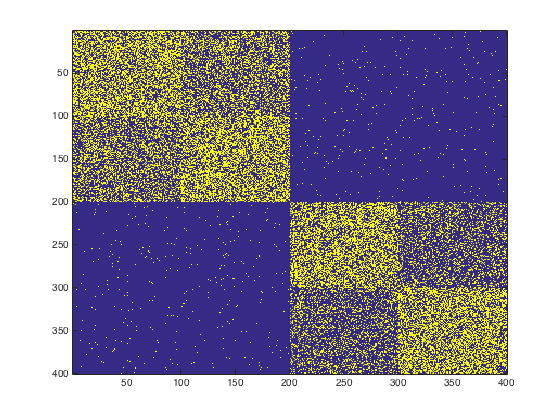} \hspace{-2em}
&\includegraphics[width=.2\textwidth]{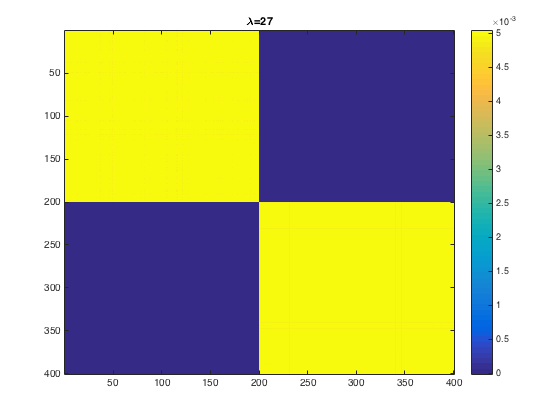} \hspace{-2em}
&\includegraphics[width=.2\textwidth]{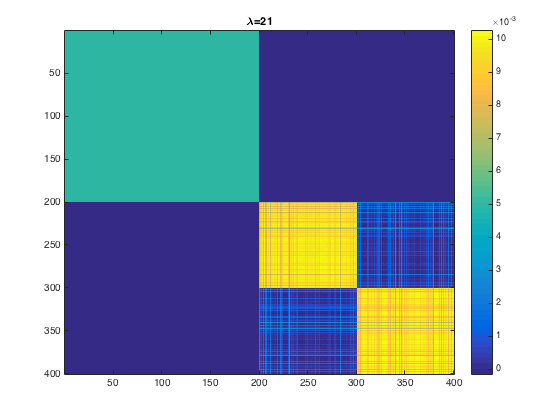} \hspace{-2em}
&\includegraphics[width=.2\textwidth]{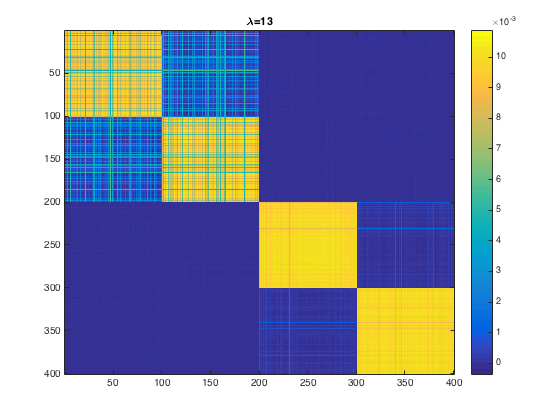} \\
(a) $B$ &(b) Adjacency &(c) $\lambda = 27$ &(d) $\lambda = 21$ &(e) $\lambda = 13$.
\end{tabular}
\caption{Solution matrices with various choices of $\lambda$.}
\label{fig:hier}
\end{figure*}

In this sketch we develop the sufficient conditions with a certain construction of the dual certificate which guarantees $X_0$ to be the optimal solution. We derive the main conditions and leave the technical details to the supplementary materials.
To start with, the KKT conditions of \eqref{eq:sdp-knownk} can be written as below.
\ba{
&\mbox{First Order Stationary } \nonumber \\
&\quad  -A - \Lambda + (1\alpha^T+\alpha1^T) + \beta I - \Gamma = 0 \label{eq:kkt-fo}\\
&\mbox{Primal Feasibility } \nonumber \\
&\quad  X\succeq 0,\ 0\leq X\le 1,\ X\bfone_n=\bfone_n, \ \tr(X)=\cl \label{eq:kkt-pf}\\
&\mbox{Dual Feasibility } \nonumber \\
&\quad  \Lambda \succeq 0,\ \ \Gamma\ge 0 \label{eq:kkt-df}\\
&\mbox{Complementary Slackness }\nonumber  \\
&\quad  \innerprod{\Lambda}{X} = 0,\ \ \Gamma \circ X=0 \label{eq:kkt-cs}
}
For \eqref{eq:sdp-lambda}, we replace $\beta$ by $\lambda$ and drop the trace constraint in the primal feasibility. Since we use $X_0$ as the primal construction, removing one primal feasibility condition has no impact on the other part of the proof.

%\subsection{Dual Certificate}
%From \eqref{eq:kkt-fo} we could immediately solve 
%\ba{\Lambda = -A+(1\alpha^T+\alpha 1^T)+\beta I-\Gamma
%\label{eq:pd-lambda}
%}

Consider the following primal-dual construction.
\ba{
&X_{S_k}=E_{m_k}/m_k; \quad X_{S_kS_\ell}=0, \forall k\ne \ell \label{eq:pd-x}\\
%&\Lambda = A-\frac{1}{2}(1\alpha^T+\alpha 1^T)-\beta I-\Gamma \label{eq:pd-lambda}\\
&\Lambda_{S_k} = -A_{S_k} + (1_{m_k}\alpha_{S_k}^T+\alpha_{S_k} 1_{m_k}^T)+\beta I_{m_k},  \nonumber\\
& \Lambda_{S_kS_\ell} = -(I-\frac{E_{m_k}}{m_k})A_{S_kS_\ell}(I-\frac{E_{m_\ell}}{m_\ell}) \label{eq:pd-lambda-kl}\\
&\Gamma_{S_k} = 0, \nonumber \\
& \Gamma_{S_k,S_\ell} = -A_{S_k,S_\ell} -\Lambda_{S_k,S_\ell} +(1_{m_k}\alpha_{S_\ell}^T+\alpha_{S_k}1_{m_\ell}^T)
 \label{eq:pd-gamma}\\
&\alpha_{S_k}=\frac{1}{m_k}\left( A_{S_k}\bfone_{m_k}+\phi_k\bfone_{m_k} \right) \label{eq:pd-alpha}\\
& \phi_k = -\frac{1}{2}\left( \beta+ \frac{\bfone_{m_k}^TA_{S_k}\bfone_{m_k}}{m_k} \right) \label{eq:pd-phi}
}
The first order condition Eq.~\eqref{eq:kkt-fo} is satisfied by construction.
By Eq.~\eqref{eq:pd-alpha} and \eqref{eq:pd-phi}, it can be seen that 
\bas{
\alpha_{S_k}^T\bfone_{m_k}=\frac{1}{m_k}\left( \bfone_{m_k}^TA_{S_k}\bfone_{m_k}\right) +\phi_k = \frac{\bfone_{m_k}^TA_{S_k}\bfone_{m_k}}{2m_k}-\frac{\beta}{2}
}

In view of the fact that both $\Lambda$ and $X$ are positive semi-definite, $\innerprod{\Lambda}{X}=0$ is equivalent to $\Lambda X=0$.
Now it remains to verify: 
\bas{
(a)\ \Lambda X=0; \qquad  (b)\ \Lambda \succeq 0; \qquad (c)\ \Gamma_{uv}\ge 0, \forall u,v
}
And it can be seen that (a) holds by construction. 
\paragraph{Positive Semidefiniteness of $\Lambda$}
For (b), since span$(\bfone_{S_k})\subset \text{ker}(\Lambda)$, it suffices to show that for any $u\in \text{span}(\bfone_{S_k})^\perp$, $u^T\Lambda u \ge \epsilon \|u\|^2$. Consider the decomposition $u=\sum_k u_{S_k}$, where $u_{S_k}:=u\circ \bfone_{S_k}$, and $u_{S_k}\perp \bfone_{m_k}$.
\bas{
&u^T\Lambda u = \sum_k u_{S_k}^T\Lambda_{S_k}u_{S_k}+\sum_{k \ne \ell} u_{S_k}^T\Lambda_{S_kS_\ell}u_{S_\ell}\\
=& -\sum_k u_{S_k}^TA_{S_k}u_{S_k} + \beta \sum_k u_{S_k}^Tu_{S_k} - \sum_{k \ne \ell} u_{S_k}^T A_{S_kS_\ell} u_{S_\ell}\\
=& -\sum_k u_{S_k}^T(A-P)_{S_k} u_{S_k}   \\
&\qquad \qquad -\sum_{k \ne\ell} u_{S_k}^T (A-P)_{S_kS_\ell} u_{S_\ell} + \beta \|u\|_2^2\\
=& -u^TAu + \beta \|u\|_2^2 \ge  \epsilon \|u\|^2
} 
In order to obtain a sufficient condition on $\beta$, we will use the following lemma from Theorem 5.2 of \cite{lei2015consistency}, which provides a tight bound for the spectral norm $\|A-\bE A\|$ for stochastic block models.
\begin{lemma}[\cite{lei2015consistency} Theorem 5.2]
Let $A$ be the adjacency matrix of a random graph on $n$ nodes in which edges occur independently. Set $\bE A=P=(p_{ij})$ and assume that $n\max_{ij}p_{ij}\le d$ for $d\ge c_0\log n$ and $c_0>0$. Then, for any $r>0$ there exists a constant $C=C(r,c_0)$ such that $ \|A-P\|\le C\sqrt{d}$, with probability at least $1-n^{-r}$.
\label{lem:operator_conc}
\end{lemma}
By Lemma~\ref{lem:operator_conc}, 
 a sufficient condition is to have
\ba{
	\beta =\Omega(\sqrt{n p_{\max}}) \ge \| A-P\|_2 
	\label{eq:beta_lowerbound}
}

\paragraph{Positiveness of $\Gamma$}
For (c), denote $d_i(S_k)=\sum_{j\in S_k}A_{i,j}$, which is the number of edges from node $i$ to cluster $k$, and $\bar{d}_i(S_k)=\frac{d_i(S_k)}{m_k}$. Define the average degree between two clusters as $\bar{d}(S_kS_\ell)=\frac{\sum_{i\in S_\ell}d_i(S_k)}{m_\ell}$. For $k\ne \ell$, %we plug \eqref{eq:pd-alpha} into \eqref{eq:pd-lambda-kl} and get 
%\beq{
%\bsplt{
%\Gamma_{S_kS_\ell} =& -\frac{1}{m_k}E_{m_k}A_{S_kS_\ell} - \frac{1}{m_\ell}A_{S_kS_\ell}E_{m_\ell} \\
%& + \frac{1}{m_km_\ell}E_{m_k}A_{S_kS_\ell}E_{m_\ell} \\
%& +\left( \frac{E_{S_kS_\ell}A_{S_\ell}}{m_\ell} +\frac{A_{S_k}E_{S_kS_\ell}}{m_k} \right) \\
%&+ \left( \frac{\phi_k}{m_k}+\frac{\phi_\ell}{m_\ell} \right)E_{m_k,m_\ell}
%}
%\label{eq:gamma_flat}
%}
%Therefore for 
$u\in C_k, v\in C_\ell$, we have
%\beq{
%\bsplt{
%& \Gamma_{uv}= -\bar{d}_v(S_k) - \bar{d}_u(S_\ell) + \bar{d}(S_kS_\ell) + \bar{d}_v(S_\ell) \\
%& + \bar{d}_u(S_k)  + \frac{\phi_k}{m_k} + \frac{\phi_\ell}{m_\ell}
%}
%\label{eq:gamma-expansion}
%}
%Plugging in Eq.~\eqref{eq:pd-phi}, we have 
$\Gamma_{uv}\ge 0 $ equivalent to 
\ba{
&\bar{d}_u(S_k) - \bar{d}_u(S_\ell)+\frac{1}{2}\left( \bar{d}(S_kS_\ell) -\bar{d}(S_kS_k) \right) \nonumber \\
&+ \bar{d}_v(S_\ell) - \bar{d}_v(S_k)+\frac{1}{2}\left( \bar{d}(S_kS_\ell) -\bar{d}(S_\ell S_\ell) \right)\nonumber  \\
& -\frac{\beta}{2m_\ell}- \frac{\beta}{2m_k}  \ge 0
\label{eq:positive_condition_for_each_node}
}
By Chernoff bound and union bound, we have a sufficient condition of 
%\bas{
%&P\left( \bar{d}_u(S_k) - \bar{d}_u(S_\ell)+\frac{1}{2}\left( \bar{d}(S_kS_\ell) -\bar{d}(S_kS_k) \right) \right. \\
%& \left. + \bar{d}_v(S_\ell) - \bar{d}_v(S_k)+\frac{1}{2}\left( \bar{d}(S_kS_\ell) -\bar{d}(S_\ell S_\ell) \right) \le \right. \\
%&  \left. \frac{1}{2}(B_{kk} - B_{k\ell})+\frac{1}{2}(B_{\ell\ell} - B_{k\ell} ) - \sqrt{6\log n}\left( \sqrt{\frac{p_k}{m_k}} \right. \right. \\
%& \left. \left. +\sqrt{\frac{p_\ell}{m_\ell}} \right) - \sqrt{18B_{k\ell}\log n\left(\frac{1}{m_k}+\frac{1}{m_\ell}\right)} \right) \le 4n^{-3}
%}
%We then apply union bound over all pairs of nodes and clusters, and combined with Eq.~\eqref{eq:beta_lowerbound}, 
$\Gamma_{uv}\ge 0$ for all pairs of $(u,v)$: 
%\bas{
%&\frac{1}{2}(p_k - B_{k\ell})+\frac{1}{2}(p_\ell - B_{k\ell} ) - \sqrt{6\log n}\left( \sqrt{\frac{p_k}{m_k}}+\sqrt{\frac{p_\ell}{m_\ell}} \right) \\
%& - \sqrt{18B_{k\ell}\log n\left(\frac{1}{m_k}+\frac{1}{m_\ell}\right)} -   c \frac{\sqrt{np_{\max}}}{m_{\min}} \ge 0
%}
%A sufficient condition gives 
\bas{
\delta \ge 2\sqrt{6\log n}\max_k \sqrt{\frac{B_{kk}}{m_k}} + \max_{\ell\ne k}6\sqrt{\frac{B_{k\ell}\log n}{m_{\min}}}+c\frac{np_{\max}}{\mmin}
}
A complete proof could be found in Appendix \ref{sec:proof_exact}.

%\begin{remark}
%For the effect of weak assortativity, note that $\min_{k} \left( B_{kk}-\max_{\ell \ne k} B_{k\ell} \right) \ge \min_{k} B_{kk}-\max_{k\ne \ell} B_{k\ell}$, hence the separation condition required in Theorem \ref{th:exact} is weaker than the conditions in strong assortativity literatures.
%\end{remark}

\section{Experiments}
\label{sec:exp}

First, we present a procedure for tuning $\lambda$ in~\eqref{eq:sdp-lambda} in subsection~\ref{sec:tuning}. Then, in subsection~\ref{sec:synth}  and~\ref{sec:real} we present results on simulated and real data.
 
\subsection{Tuning and substructure finding}
\label{sec:tuning}

%\begin{wrapfigure}{L}{0.6\textwidth}
 %   \begin{minipage}{0.6\textwidth}
\begin{algorithm}[H]
\begin{algorithmic}
\STATE {\bf Input}: graph $A$, number of candidates $T$;
\FOR{i = 0:T-1}
\STATE $\lambda = \exp(\frac{i}{T}\log(1+\|A\|_{op}))-1$;
\STATE $\hat{X}_{\lambda} =$ solution of \eqref{eq:sdp-lambda}.
\STATE $\theta(\lambda) = \frac{\sum_{i\le r_\lambda}\sigma_i(X_\lambda)}{\tr(\hat{X}_\lambda)}$;
\ENDFOR
\STATE $\hat{\lambda} = \arg\max_{\lambda} \theta(\lambda)$;
\STATE {\bf Output}: $\hat{X}_{\hat{\lambda}}$, $\hat{r}=[\tr(\hat{X}_{\hat{\lambda}})]$;
\end{algorithmic}
\caption{Semidefinite Program  with Unknown $r$ (SPUR)}
\label{alg:sdp_tuning}
\end{algorithm}
%\end{minipage}
%\end{wrapfigure}

As shown in Proposition~\ref{prop:lambda_upper}, choice of $\lambda$ should not exceed the operator norm of the observed network. Therefore we do a grid search for $\lambda$ from 0 to $\|A\|_{op}$ in log scale.
For each candidate $\lambda$, we solve \eqref{eq:sdp-lambda} and get the corresponding solution $\hat{X}_\lambda$.
The estimated number of clusters is defined as $r_\lambda = [\tr(\hat{X}_\lambda)]$, where $[\cdot]$ represent the rounding operator. Let $\sigma_i(X)$ be the $i$-th eigenvalue of $X$. We then pick the solution which maximizes the proportion of leading eigenvalues $\hat{\lambda} = \arg\max_\lambda \sum_{i\le r_\lambda}\sigma_i(\hat{X}_\lambda)/\tr(\hat{X}_\lambda)$. This fraction calculates the proportion of leading eigenvalues in the entire spectrum. If it equals to one, then the solution is low rank. The algorithm is summarized in Algorithm~\ref{alg:sdp_tuning}. 
In the experiments, for scalability concerns we fix a smaller range and search over the range $0.1\sqrt{\bar{d}}$ to $2\sqrt{\bar{d}}$, where $\bar{d}$ denotes the average degree. 

In theory, when $\lambda$ lies in the interval specified by Corollary~\ref{cor:estk} exact recovery is possible. Yet, in practice, solutions with different choices of $\lambda$, even outside of the theoretical range, still gives us some useful information about the sub-structures of the network. Figure~\ref{fig:hier} shows a probability matrix which has large separation into two big clusters and each further splits into two smaller clusters with different separations. With a larger $\lambda$ it returns an under estimated $r$, but consistent to the hierarchical structure in the original network. In this vein, the tuning method provides a great way to do exploratory analysis of the network.

\begin{figure}[h]
\centering
\begin{tabular}{cc}
\hspace{-2em}   \raisebox{.1\height}{\includegraphics[width=.24\textwidth, height= 3.4 cm]{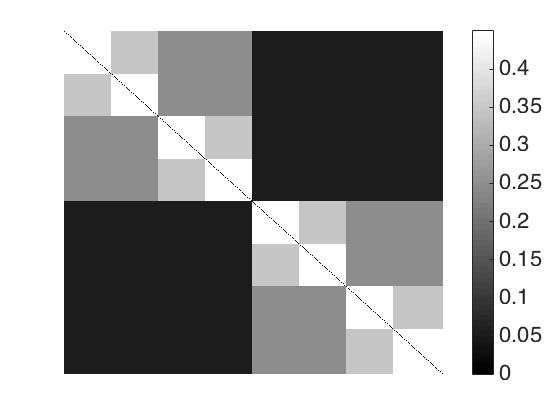} } 
&\hspace{-2.5em} \includegraphics[width=.24\textwidth, height= 3.6 cm]{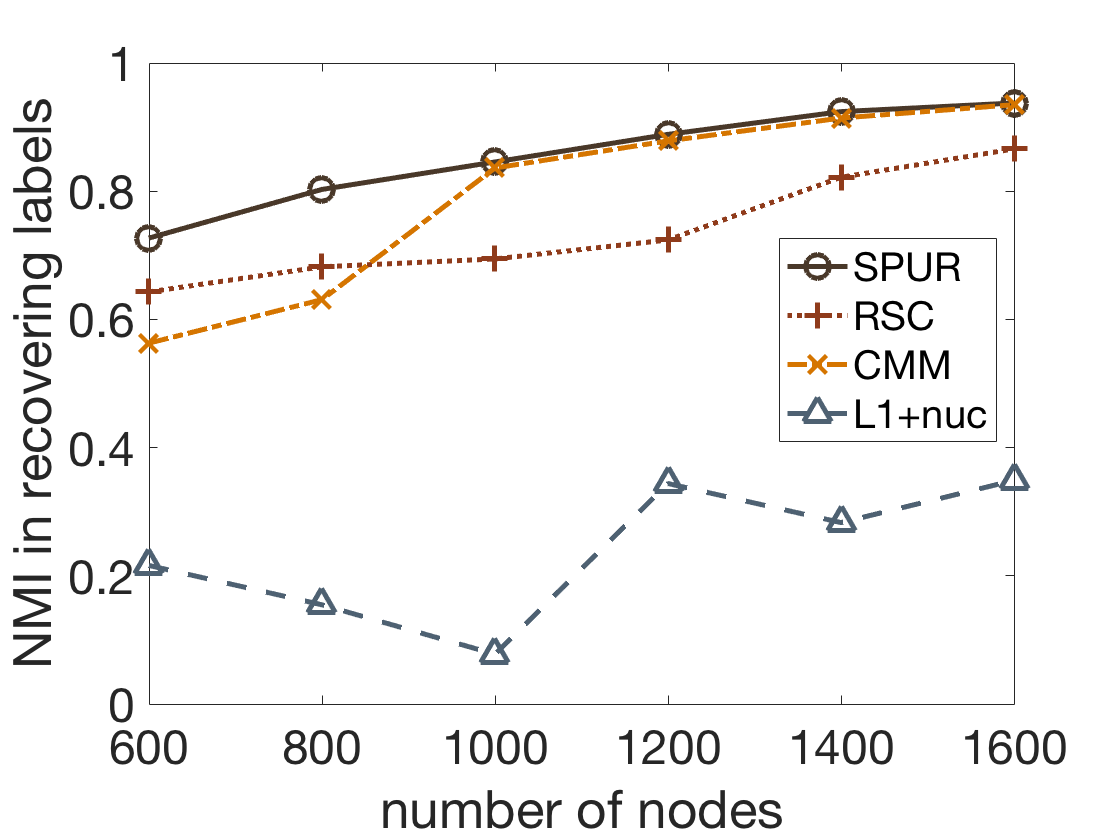} \\
(a) Expectation of network: & (b) NMI;\\
\end{tabular}
\caption{The expectation matrix and NMI used for the known $r$ setting. }
\label{fig:knownk}
\end{figure}

\subsection{Synthetic data}
\label{sec:synth}
We present our simulation results in three parts - known $r$, increasing $r$ and unknown $r$. We report the normalized mutual information (NMI) of predicted label and ground truth membership, and the accuracy of estimating $r$. For each experiment, the average over 10 replicates is reported.

\paragraph{Known number of clusters}
We compare the NMI of \sdp against some state-of-the-art methods, including Regularized Spectral Clustering (RSC) \cite{amini2013pseudo}, and two convex relaxations which do not require $r$ as input to the optimization: convexified modularity maximization (CMM) in \cite{chen2015convexified}; and the $\ell_1$ plus nuclear norm penalty method proposed in \cite{chen2014improved} (\lnuc). In this setting, we use \eqref{eq:sdp-knownk} directly which does not involve any tuning. In contracst, due to the hierarchical structure of the network, the default values for the tuning parameters in both methods would only be able to recover the lowest level of hierarchy, which consists of two clusters. Hence for a fair comparison, we try a grid search for those tuning parameters and choose the one that gives largest eigengap between the $r$-th and $r+1$th eigenvalues of the clustering matrices.
The expectation of the network generated is shown in the left panel of 
Figure~\ref{fig:knownk}. The right panel shows that the proposed method outperforms the competing methods. %The other two convexified methods do not work well for this type of structure. Although it is uniformly strongly assortative, the separation gets smaller as one goes down in the hierarchy. In this case, CMM and \lnuc learn a clustering matrix with 2 clusters rather than 8.

\begin{figure}[h]
\centering
\includegraphics[width=.28\textwidth]{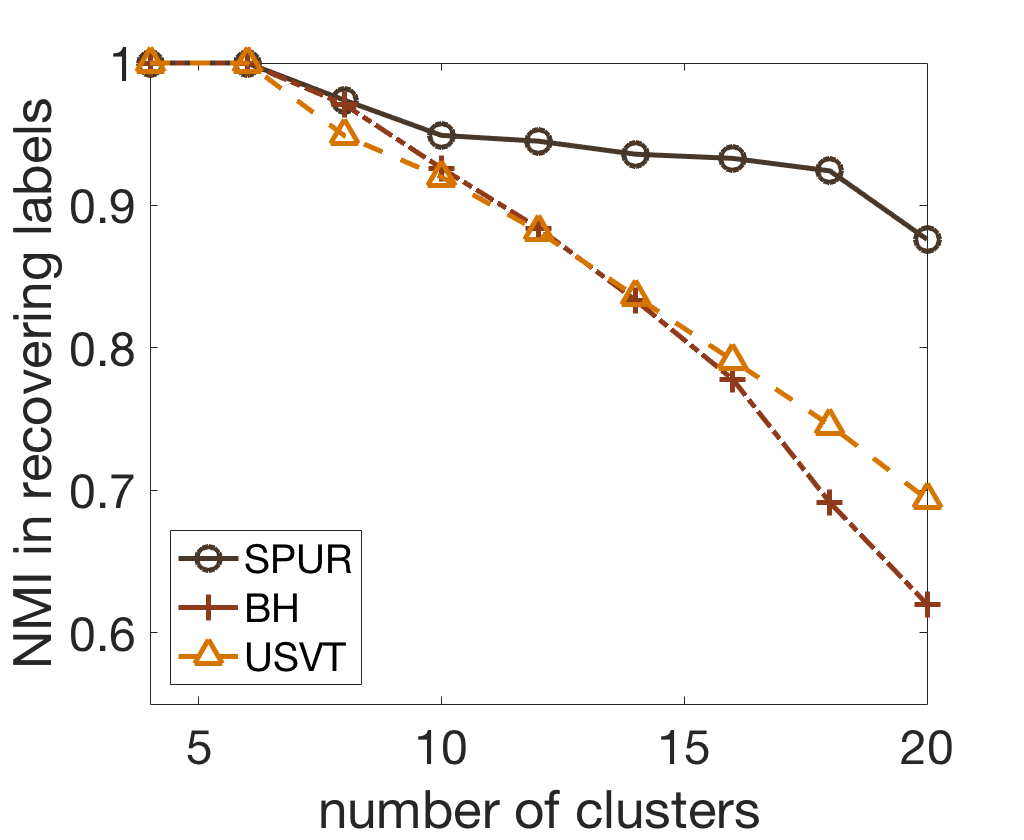} 
\caption{NMI under planted partition model with increasing (unknown) number of clusters.}
\label{fig:increasek}
\end{figure}

\begin{figure*}[t]
\centering
\begin{tabular}{crc}
%\hspace{-2em}   \raisebox{.1\height}{\includegraphics[width=.29\textwidth, height= 3.85 cm]{fig/Ap_knownk.png} } 
%&\hspace{-2.5em} \includegraphics[width=.31\textwidth, height= 4.1 cm]{fig/knownk_tuning_earthcolor.png} \hspace{-1em}
%&\includegraphics[width=.31\textwidth, height= 4.1 cm]{fig/increasek_camera.png} \\
%(a) known $r$: & NMI  & (b) NMI (varying $r$);\\
\raisebox{.1\height}{\hspace{-2em}\includegraphics[width=.3\textwidth, height= 3.85 cm]{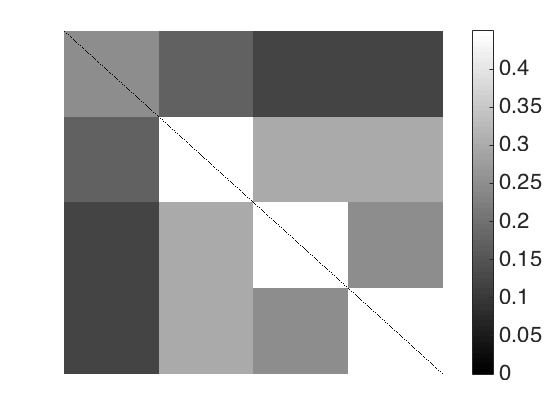}  \hspace{-1em}} 
& \multicolumn{2}{c}{\hspace{-3em} \includegraphics[width=.6\textwidth, height= 4.1 cm]{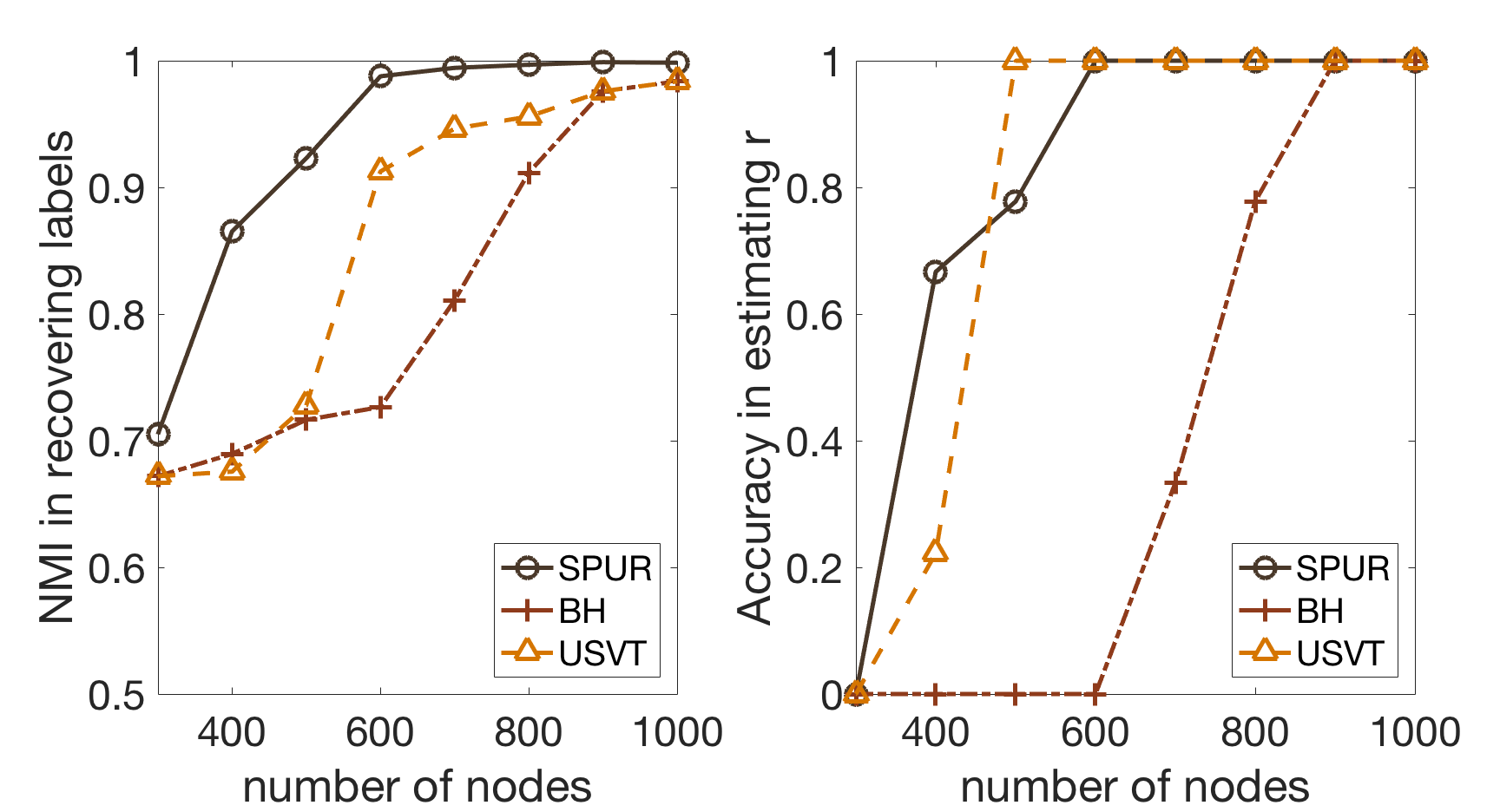}} \\
(a) balanced setting with unknown $r$: & \mbox{               }NMI & Accuracy of $r$.\\
  \raisebox{.1\height}{\hspace{-2em}\includegraphics[width=.3\textwidth, height= 3.85 cm]{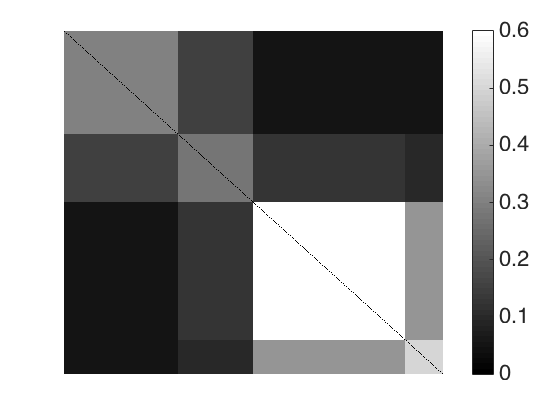} \hspace{-1em}} 
& \multicolumn{2}{c}{\hspace{-3em} \includegraphics[width=.6\textwidth, height= 4.1 cm]{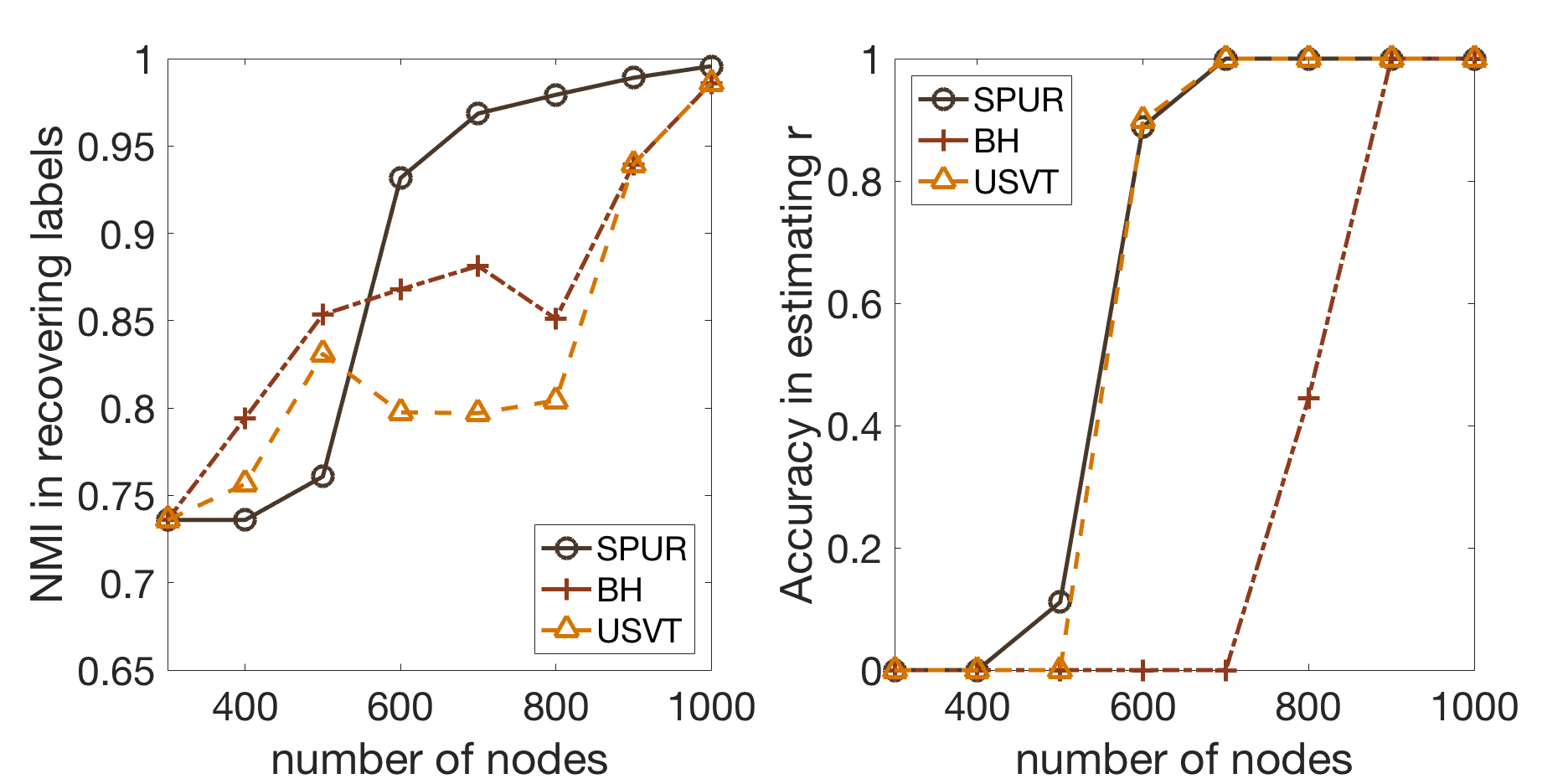} \hspace{-1em}} \\
(b) unbalanced setting with unknown $r$: & \mbox{               }NMI & Accuracy of $r$.
\end{tabular}
\caption{%(a) shows the expectation matrix and NMI used for the known $r$ setting. (b) shows the NMI under planted partition model with increasing (unknown) number of clusters. 
The first row shows weakly assortative models with balanced cluster sizes and the corresponding NMI and accuracy in estimating $r$; the second row shows those for unbalanced cluster sizes.
}
\label{fig:weak_imb}
\end{figure*}

\begin{figure}[H]
\centering
\includegraphics[width=.18\textwidth, height=2.3cm]{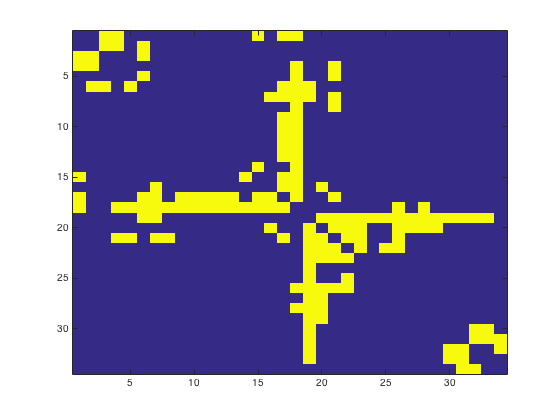} 
 \includegraphics[width=.28\textwidth, height=2.8cm]{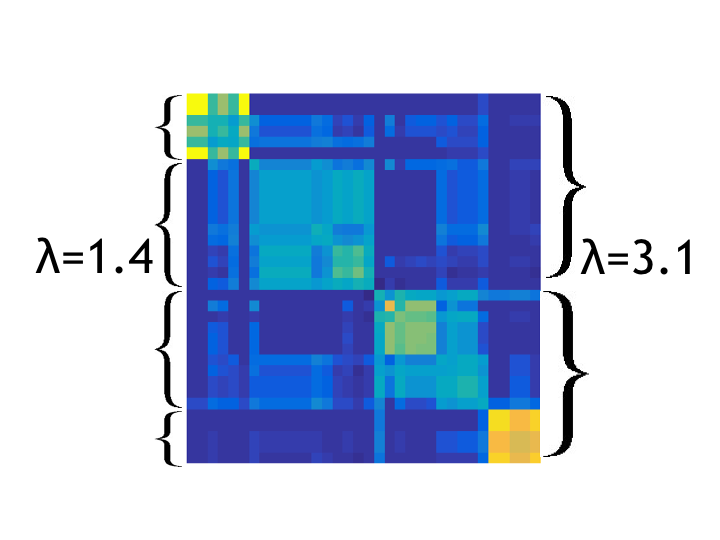} 
\caption{Adjacency matrix and predicted $X$ for karate club dataset; ordered by predicted labels. }
\label{fig:real_recover}
\end{figure}

\paragraph{Increasing number of clusters}
In this experiment, we fix the number of nodes as 400 and increase the number of clusters from 4 to 20. With each given $r$ we generate the graph with $B_{kk}=0.6, B_{k\ell}=0.1, \forall k\ne\ell$ and $\mmax/\mmin=4$, then run the various estimation algorithms same as in previous experiment. It is shown in \ref{fig:increasek} that as number of clusters increases, all methods deteriorate, but the performance for SPUR declines slower than the others.

\paragraph{Unknown number of clusters}
In this experiment, we carry out two synthetic experiments for weakly assortative graphs for both balanced and unbalanced cluster sizes. 
We generate the network with expectation matrices shown in the leftmost column of Figure~\ref{fig:weak_imb}, and show the NMI of predicted labels with ground truth labels, and the fraction of returning the correct $r$, for both balanced (Figure~\ref{fig:weak_imb}-(a)) and unbalanced  (Figure~\ref{fig:weak_imb}-(b)) settings.
We run \sdp %with the tuning procedure described in Section \ref{sec:tuning}, 
and compare the result with 1) the Bethe-Hessian estimator (BH) in \cite{le2015estimating}, in particular BHac (which has been shown to perform better for unbalanced settings), 2) USVT in \cite{chatterjee2015matrix}.%, 3) CMM in \cite{chen2015convexified}, and \lnuc in~\cite{chen2014improved}. 
%It is worth pointing out that, although CMM returns consistent clustering matrix without knowing $r$, it requires $r$ for the final clustering on $X$. In the experiments, we use the largest eigen gap of $\hat{X}$ to pick $r$ as prescribed by the authors in another paper~\cite{chen2014improved}. 
For all competing methods, we run spectral clustering with the estimated $r$ to estimate the cluster memberships.
As we can see here, \sdp has a better accuracy in label recovery than competing methods. \sdp also achieves accurate cluster number faster than competing methods. 

\setlength{\tabcolsep}{2pt}
\begin{table}[h]
\caption{Estimated number of clusters for real networks.} \label{tab:real}
\begin{center}
\begin{tabular}[b]{c|cccccc}
		\toprule
		Datasets & Truth & SDP & BH & USVT  & CMM \\
		\midrule
		College Football & 12 & 13 & 10   & 10  & 10\\
		Political Books & 3 & 3 & 4 &  4  & 2 \\
		Political Blogs &2 & 3 & 8 &  3  &2\\
		Dolphins & 2 & 5 & 2 & 4  & 7 \\
		Karate  &2 & 2 &2& 2  &2\\
		\bottomrule
	\end{tabular}
\end{center}
\end{table}
\subsection{Real Datasets}
\label{sec:real}
We apply the proposed method on several real world data sets\footnote{All datasets used here are available at \url{http://www-personal.umich.edu/~mejn/netdata/}.}: the college football dataset \cite{girvan2002community}, the political books, political blogs \cite{adamic2005political}, dolphins and karate club \cite{zachary1977information} datasets.
We compare the performance of \sdp with BH, CMM and USVT in Table \ref{tab:real}. As seen from~\cite{le2015estimating}, most algorithm correctly finds $r$ for about 2 or 3 of these networks. It is also worth pointing out that this typically happens because  different techniques finds different clusterings of the hidden substructures~\cite{bickel2016hypothesis}. We will now show one such substructure we found in the Karate club data.

%estimation by real eigenvalues of non-backtracking matrix (NB), the Bethe-Hessian method (BH) \cite{le2015estimating}, network cross-validation (NCV) \cite{chen2016network}, variational likelihood BIC (VLH) \cite{wang2015likelihood} and show the comparison in Table \ref{tab:real}.

%From the result shown in Table \ref{tab:real}, we see that the proposed method works better on larger and well-separated graphs. And it is also good at finding small clusters and sub-structures within large clusters.
%Take the college football dataset as an example, this dataset is a representation of the schedule of Division I games for the 2000 season: vertices in the graph represent teams and edges represent regular-season games between the two teams they connect. The teams are divided into conferences containing around 8 to 12 teams each. Games are more frequent between members of the same conference than between members of different conferences, with teams playing an average of about seven intra-conference games and four inter-conference games in the 2000 season. Inter-conference play is not uniformly distributed; teams that are geographically close to one another but belong to different conferences are more likely to play one another than teams separated by large geographic distances. As shown in Figure \ref{fig:real_recover}, our method correctly recovers all 12 conferences, whereas most competing methods overlook the two smallest clusters. 
Figure \ref{fig:real_recover} shows the adjacency matrix and $\xhat$ for the Karate club data set. For $\lambda=3.1$, we find two clusters, whereas for $\lambda=1.4$, we find 4 clusters, which are further subdivisions of the first level. While our tuning method picks up $\lambda=3.1$ ($r=2$) based on the scoring, we show the substructure for $\lambda=1.4,r=4$ in Figure \ref{fig:real_recover}. %Figure \ref{fig:real_recover} 
The left panel shows the adjacency matrix of the Karate club data ordered according to the clusters obtained with $\lambda=1.4$.  The right panel of Figure \ref{fig:real_recover} shows finer substructure of $\xhat$; as suggested by the adjacency matrix, within each group there are two small clique like groups at the two corners, and the hubs from each group.

\section{Conclusion}
We present SPUR, a SDP-based algorithm which provably learns the number of clusters $r$ in a SBM under the weakly assortative setting. Our approach does not require the knowledge of model parameters, and foregoes the added tuning step used by existing SDP approaches for unequal size clusters when $r$ is known. For unknown $r$, the tuning in the objective provides guidance in exploring the finer sub-structure in the network. Simulated and real data experiments show that SPUR performs comparably or better than state-of-the-art approaches.
%We propose an algorithm named SPUR, which is shown to perform comparably or better than state-of-the-art approaches.

\subsubsection*{Acknowledgements}
PS was partially supported by NSF grant DMS 1713082.
%\clearpage
\bibliographystyle{plain}
\bibliography{bib_exact}

\clearpage
\onecolumn
\appendix
%\section{Accompanying Lemmas}
%The following lemma from Theorem 5.2 of \cite{lei2015consistency} provides a strong bound for spectral norm $\|A-\bE A\|$ for stochastic block model.
%\begin{lemma}[\cite{lei2015consistency} Theorem 5.2]
%Let $A$ be the adjacency matrix of a random graph on $n$ nodes in which edges occur independently. Set $\bE A=P=(p_{ij})$ and assume that $n\max_{ij}p_{ij}\le d$ for $d\ge c_0\log n$ and $c_0>0$. Then, for any $r>0$ there exists a constant $C=C(r,c_0)$ such that $ \|A-P\|\le C\sqrt{d}$, with probability at least $1-n^{-r}$.
%\label{lem:operator_conc}
%\end{lemma}

%%%%%%%%%%%%%%
% MAIN PROOF
%%%%%%%%%%%%%%
\section{Proof of Theorem \ref{th:unspecified} and \ref{th:exact}}
\label{sec:proof_exact}

\begin{proof}[Proof of Theorem \ref{th:exact}]
The construction \eqref{eq:pd-lambda-kl}-\eqref{eq:pd-phi} together with $X_0$ is a primal-dual certificate, if  \eqref{eq:kkt-fo}-\eqref{eq:kkt-cs} are satisfied. In view of the fact that both $\Lambda$ and $X$ are positive semi-definite, $\innerprod{\Lambda}{X}=0$ is equivalent to $\lambda X=0$. We need to check the following:
\begin{itemize}
\item[(a)] $\Lambda X=0$;
\item[(b)] $\Lambda \succeq 0$;
\item[(c)] $\Gamma_{uv}\ge 0, \forall u,v$.
\end{itemize}

Note that span$(X)$=span$(\bfone_{S_k})$, therefore we only need to show $\Lambda \bfone_{S_k}=0, \forall k\in[r]$. Or equivalently $\Lambda_{S_k} \bfone_{m_k}=0$ and $\Lambda_{S_kS_\ell} \bfone_{m_\ell}=0$. The latter holds by \eqref{eq:pd-lambda-kl}. For the former, recall that $\alpha_{S_k}^T\bfone_{m_k}=\frac{1}{m_k}\left( \bfone_{m_k}^TA_{S_k}\bfone_{m_k}\right) +\phi_k$.
\bas{
0=& \Lambda_{S_k} \bfone_{m_k} = -A_{S_k}\bfone_{m_k} + (\bfone_{m_k}\alpha_{S_k}^T\bfone_{m_k}+\alpha_{S_k} \bfone_{m_k}^T\bfone_{m_k}) + \beta \bfone_{m_k}\\
=& -A_{S_k}\bfone_{m_k} + \left( \frac{\bfone_{m_k}^TA_{S_k}\bfone_{m_k}}{m_k} + \phi_k \right) \bfone_{m_k} + A_{S_k}\bfone_{m_k} + \phi_k\bfone_{m_k} + \beta \bfone_{m_k}\\
=& \left( \frac{\bfone_{m_k}^TA_{S_k}\bfone_{m_k}}{m_k}  \right) \bfone_{m_k} + 2\phi_k\bfone_{m_k} + \beta \bfone_{m_k}
}
The equation holds by taking 
\ba{
\phi_k=-\frac{1}{2}\left( \beta+ \frac{\bfone_{m_k}^TA_{S_k}\bfone_{m_k}}{m_k} \right).
\label{eq:choice_of_phik}
}

\paragraph{Positive Semidefiniteness of $\Lambda$}
For (b), since span$(\bfone_{S_k})\subset \text{ker}(\Lambda)$, it suffices to show that for any $u\in \text{span}(\bfone_{S_k})^\perp$, $u^T\Lambda u \ge \epsilon \|u\|^2$. Consider the decomposition $u=\sum_k u_{S_k}$, where $u_{S_k}:=u\circ \bfone_{S_k}$, and $u_{S_k}\perp \bfone_{m_k}$.
\bas{
u^T\Lambda u =& \sum_k u_{S_k}^T\Lambda_{S_k}u_{S_k}+\sum_{k \ne \ell} u_{S_k}^T\Lambda_{S_kS_\ell}u_{S_\ell}\\
=& -\sum_k u_{S_k}^TA_{S_k}u_{S_k} + \beta \sum_k u_{S_k}^Tu_{S_k} - \sum_{k \ne \ell} u_{S_k}^T A_{S_kS_\ell} u_{S_\ell}\\
=& -\sum_k u_{S_k}^T(A-P)_{S_k} u_{S_k}  - \sum_{k \ne\ell} u_{S_k}^T (A-P)_{S_kS_\ell} u_{S_\ell} + \beta \|u\|_2^2\\
=& -u^TAu + \beta \|u\|_2^2 \ge  \epsilon \|u\|^2
} 

In order to have $\beta \ge \| A-P\|_2 $, using Lemma~\ref{lem:operator_conc}, we propose the following  sufficient condition: 
\ba{
\beta =\Omega(\sqrt{n p_{\max}})\ge \| A-P\|_2 
\label{eq:beta_lowerbound_app}
}
%By Lemma~\ref{lem:operator_conc}, we have $\beta=\Omega(\sqrt{n p_{\max}})$.

\paragraph{Positiveness of $\Gamma$}
For (c), denote $d_i(S_k)=\sum_{j\in S_k}A_{i,j}$, which is the number of edges from node $i$ to cluster $k$, and $\bar{d}_i(S_k)=\frac{d_i(S_k)}{m_k}$. Define the average degree between two clusters as $\bar{d}(S_kS_\ell)=\frac{\sum_{i\in S_\ell}d_i(S_k)}{m_\ell}$. For $k\ne \ell$, we plug \eqref{eq:pd-alpha} into \eqref{eq:pd-lambda-kl} and get 
\beq{
\bsplt{
\Gamma_{S_kS_\ell} =& -A_{S_kS_\ell} + (I-\frac{1}{m_k}E_{m_k})A_{S_kS_\ell}(I-\frac{1}{m_\ell}E_{m_\ell}) + \frac{1}{m_k}\left( A_{S_k}\bfone_{m_k}+\phi_k\bfone_{m_k} \right) \bfone_{m_\ell}^T \\
&\qquad \qquad \qquad+ \bfone_{m_k}\frac{1}{m_\ell}\left( \bfone_{m_\ell}^TA_{S_\ell} +\phi_\ell \bfone_{m_\ell}^T \right) \\
=& -\frac{1}{m_k}E_{m_k}A_{S_kS_\ell} - \frac{1}{m_\ell}A_{S_kS_\ell}E_{m_\ell} + \frac{1}{m_km_\ell}E_{m_k}A_{S_kS_\ell}E_{m_\ell} +\\
&\qquad \qquad \left( \frac{E_{S_kS_\ell}A_{S_\ell}}{m_\ell} +\frac{A_{S_k}E_{S_kS_\ell}}{m_k} \right) + \left( \frac{\phi_k}{m_k}+\frac{\phi_\ell}{m_\ell} \right)E_{m_k,m_\ell}
}
\label{eq:gamma_flat}
}
Therefore for $u\in C_k, v\in C_\ell$, we have
\beq{
\bsplt{
& \Gamma_{uv}= -\bar{d}_v(S_k) - \bar{d}_u(S_\ell) + \bar{d}(S_kS_\ell) + \bar{d}_v(S_\ell) + \bar{d}_u(S_k) + \frac{\phi_k}{m_k} + \frac{\phi_\ell}{m_\ell}
}
\label{eq:gamma-expansion}
}

Plugging in Eq~\eqref{eq:choice_of_phik}, we have $\Gamma_{uv}\ge 0 $ equivalent to 
\ba{
\bar{d}_u(S_k) - \bar{d}_u(S_\ell)+\frac{1}{2}\left( \bar{d}(S_kS_\ell) -\bar{d}(S_kS_k) \right)+ \bar{d}_v(S_\ell) - \bar{d}_v(S_k)+&\frac{1}{2}\left( \bar{d}(S_kS_\ell) -\bar{d}(S_\ell S_\ell) \right) \\
& -\frac{\beta}{2m_\ell}- \frac{\beta}{2m_k}  \ge 0
\label{eq:positive_condition_for_each_node}
}
By Chernoff bound, we have
\bas{
& P \left( \bar{d}_u(S_k) \le B_{kk}-\sqrt{\frac{6 B_{kk}\log n }{m_k}} \right) \le n^{-3} \\
& P\left( \bar{d}_u(S_\ell) \ge B_{k\ell} + \sqrt{\frac{18B_{k\ell}\log n }{m_\ell}} \right) \le n^{-3} \\
& P\left( \bar{d}(S_kS_k) \ge B_{kk} + \sqrt{\frac{18B_{kk}\log n }{m_k(m_k-1)}} \right) \le n^{-3} \\
& P\left( \bar{d}(S_kS_\ell) \le B_{k\ell} - \sqrt{\frac{6B_{k\ell}\log n }{m_km_\ell}} \right) \le n^{-3} 
}
Apply union bound we have,
\bas{
&P\left(\bar{d}_u(S_k) - \bar{d}_u(S_\ell)+\frac{1}{2}\left( \bar{d}(S_kS_\ell) -\bar{d}(S_kS_k) \right)+ \bar{d}_v(S_\ell) - \bar{d}_v(S_k)+\frac{1}{2}\left( \bar{d}(S_kS_\ell) -\bar{d}(S_\ell S_\ell) \right) \le \right. \\
& \left. \frac{1}{2}(B_{kk} - B_{k\ell})+\frac{1}{2}(B_{\ell\ell} - B_{k\ell} ) - \sqrt{6\log n}\left( \sqrt{\frac{B_{kk}}{m_k}}+\sqrt{\frac{p_\ell}{m_\ell}} \right) - \sqrt{18B_{k\ell}\log n\left(\frac{1}{m_k}+\frac{1}{m_\ell}\right)} \right) \le 4n^{-3}
}
We then apply union bound over all pairs of nodes and clusters, and combined with Eq.~\eqref{eq:beta_lowerbound}, $\Gamma_{uv}\ge 0$ for all pairs of $(u,v)$ if 
\bas{
&\frac{1}{2}(B_{kk} - B_{k\ell})+\frac{1}{2}(B_{\ell\ell} - B_{k\ell} ) - \sqrt{6\log n}\left( \sqrt{\frac{B_{kk}}{m_k}}+\sqrt{\frac{B_{\ell\ell}}{m_\ell}} \right) \\
& \quad - \sqrt{18B_{k\ell}\log n\left(\frac{1}{m_k}+\frac{1}{m_\ell}\right)} -   c \frac{\sqrt{np_{\max}}}{m_{\min}} \ge 0
}
The proof follows by relaxing $B_{kk} - B_{k\ell}$ with the minimum over all clusters.
\end{proof}

%\begin{proof}[Proof of Theorem~\ref{th:unspecified}]
For the proof of Theorem~\ref{th:unspecified}, we use the same dual certificate construction Eq.~\eqref{eq:pd-lambda-kl}-\eqref{eq:pd-phi}, with $\beta=\lambda$. The existence of the primal-dual certificate is guaranteed by the proof of Theorem~\ref{th:exact}.
%From the proof of Theorem~\ref{th:exact}, the existence is guaranteed by Eq.~\eqref{eq:beta_lowerbound} and Eq.~\eqref{eq:beta_upperbound}. 
%\end{proof}

\section{Proof of Proposition \ref{prop:lambda_upper}}
\label{app:lambda_upper}
\begin{proof}
When $\lambda\ge \|A\|_{op}$, $\tilde{A} = A-\lambda I \preceq 0$. From the constraint we know that $X\succeq 0$, and has at least one eigenvalue 1 with eigenvector $\bfone/\sqrt{n}$. Consider an eigen-decomposition $X = \frac{1}{n}\bfone \bfone^T +\sum_{i=2}^n s_iu_iu_i^T$ where $s_i\ge 0$. Then the objective is
\bas{
\innerprod{\tilde{A}}{X} = \bfone^T \tilde{A}\bfone/n+\sum_i s_i u_i^T \tilde{A}u_i
}
Note that $s_i\ge 0$ and $\tilde{A}\preceq 0$, so the above objective is maximized when $s_i=0,\forall i\ge 2$. Therefore $X^* = \bfone \bfone^T/n$.
\end{proof}

\end{document}